\documentclass{article}

\usepackage{times}
\usepackage{graphicx} 
\usepackage{subfigure} 

\usepackage{natbib}

\usepackage{algorithm}
\usepackage{algorithmic}

\usepackage{hyperref}

\usepackage{eso-pic}
\usepackage{tikz}

\usepackage{url}

\usepackage{amsthm}
\usepackage{amsmath}
\usepackage{amsfonts}
\usepackage{xspace}
\usepackage{multicol}

\newtheorem{thm}{Theorem}[section]
\newtheorem{cor}[thm]{Corollary}
\newtheorem{lem}[thm]{Lemma}

\newtheorem{obs}[thm]{Observation}
\theoremstyle{definition}
\newtheorem{defn}[thm]{Definition}

\theoremstyle{remark}

\numberwithin{equation}{section}
\def\R{\mathbb{R}}
\def\E{\mathbb{E}}

\def\eps{\varepsilon}

\def\argmax{\operatorname{argmax}}

\def\R{\mathbb{R}}

\def\philin{{\phi_{\operatorname{lin}}}}
\def\rst{{\operatorname{reset}}}
\def\fdbck{{\operatorname{feedback}}}

\def\advnce{{\operatorname{advance}}}

\def\confballone{{\sc ConfidenceBall}$_1$\xspace}
\def\confballtwo{{\sc ConfidenceBall}$_2$\xspace}
\def\eqdef{:=}
\def\Umean{{U^{\operatorname{av}}}}
\def\Uchoice{{U^{\operatorname{choice}}}}
\def\Rmean{{R^{\operatorname{av}}}}

\def\Umeant{{U_t^{\operatorname{av}}}}
\def\Uchoicet{{U_t^{\operatorname{choice}}}}
\def\Rmeant{{R_t^{\operatorname{av}}}}
\def\Rchoicet{{R_t^{\operatorname{choice}}}}

\newcommand{\newtau}{\rho}
\newcommand{\secref}[1]{Sec.~\ref{#1}}

\newcommand{\dotprod}[2]{\left \langle #1, #2 \right \rangle}

\newcommand{\kc}[1]{}
\newcommand{\tj}[1]{}

\newcommand{\InF}{{\bf IF}}
\newcommand{\BTMB}{{\bf BTMB}}
\newcommand{\MultiSbm}{{\bf MultiSBM}}
\newcommand{\Doubler}{{\bf Doubler}}
\newcommand{\DoubleSbm}{{\bf Sparring}}

\begin{document} 

\title{Reducing Dueling Bandits to Cardinal Bandits}
\author{
Nir Ailon \\
CS Dept. Technion \\
Haifa, Israel \\
\texttt{nailon@cs.technion.ac.il}
\and
Zohar Karnin \\
Yahoo! Research \\
Haifa, Israel \\
\texttt{zkarnin@gmail.com}
\and
Thorsten Joachims \\
Cs Dept. Cornell \\
Ithaca, NY \\
\texttt{tj@cs.cornell.edu}
}
\maketitle


\begin{abstract} 
We present algorithms for reducing the Dueling Bandits problem to the conventional (stochastic) Multi-Armed Bandits problem. The Dueling Bandits problem is an online model of learning with ordinal feedback of the form ``A is preferred to B'' (as opposed to cardinal feedback like ``A has value 2.5''), giving it wide applicability in learning from implicit user feedback and revealed and stated preferences. In contrast to existing algorithms for the Dueling Bandits problem, our reductions -- named $\Doubler$, $\MultiSbm$ and $\DoubleSbm$ -- provide a generic schema for translating the extensive body of known results about conventional Multi-Armed Bandit algorithms to the Dueling Bandits setting.   
For $\Doubler$ and $\MultiSbm$ we prove regret upper bounds in both finite and infinite settings, and conjecture about the performance of $\DoubleSbm$ which empirically outperforms the other two as well as previous algorithms in our experiments.  In addition, we provide the first almost optimal regret bound in terms of second order terms, such as the differences between the values of the arms.
\end{abstract}

\section{Introduction}\label{sec:intro}

When interacting with an online system, users reveal their preferences through the choices they make. Such a choice -- often termed implicit feedback -- may be the click or tap on a particular link in a web-search ranking, or watching a particular movie among a set of recommendations. Connecting to a classic body of work in econometrics and empirical work in information retrieval \cite{Joachims/etal/07a}, such implicit feedback is typically viewed as an ordinal preference between alternatives (i.e., ``A is better than B''), but it does not provide reliable cardinal valuations (i.e., ``A is very good, B is mediocre'').

To formalize the problem of learning from preferences, we consider the following interactive online learning model, which we call the {\bf Utility-Based Dueling Bandits Problem (UBDB)} similar to \cite{DBLP:journals/jcss/YueBKJ12,DBLP:conf/icml/YueJ11}. At each iteration $t$, the learning system presents \emph{two} actions  $x_t,y_t\in X$ to the user, where $X$ is the set (either finite or infinite) of possible actions.  Each of the two actions has an associated random reward (or utility) for the user, which we denote by
 $u_t$ and $v_t$, respectively.  The quantity $u_t$ (resp. $v_t$) is drawn from a distribution that depends on $x_t$  (resp. $y_t$) only.     We assume these utilities are in $[0,1]$. 
The learning system is rewarded the average utility $\Umeant = (u_t+v_t)/2$ of the two actions it presents, but \emph{it does not observe this reward}.  Instead, it only observes the user's binary choice among the two alternative actions $x_t, y_t$, which depends on the respective utilities $u_t$ and $v_t$. In particular, we model the observed choice as a $\{0,1\}$-valued random variable $b_t$ distributed as
\begin{eqnarray}
 \Pr[b_t=0 | (u_t, v_t)] &=& \phi(u_t, v_t) \nonumber \\
Pr[b_t=1 | (u_t, v_t)] &=& \phi(v_t, u_t)\ , \label{defchoice}
\end{eqnarray}
where $\phi : [0,1]\times [0,1]\mapsto [0,1]$ is a \emph{link} function.  Clearly, the link function has to satisfy $\phi(A,B)+\phi(B,A) = 1$. Below we concentrate on linear link functions (defined in \secref{sec:def}).
The binary choice is interpreted as a stochastic preference response between the \emph{left alternative} $x_t$ (if $b_t=0$) and the \emph{right alternative} $y_t$ (if $b_t=1$). The utility $\Umean$ captures the overall latent user experience from the pair of alternatives. A concrete example of this UBDB game is learning for web search, where $X$ is a set of ranking functions among which the search engine selects two for each incoming query; the search engine then presents an interleaving \cite{Chapelle/etal/12a} of the two rankings, from which it can sense a stochastic preference between the two ranking functions based on the user's clicking behavior.


The purpose of this paper is to  show how  UBDB  can be reduced to the conventional (cardinal) stochastic Multi-Armed Bandit (MAB) problem\footnote{One armed bandit is a popular slang for slot machines in casinos, and the MAB game describes the problem faced by a gambler who can choose one machine to play at each instance.}, which has been studied since 1952~\cite{citeulike:761755}.
In MAB, the system chooses only a single action $x_t\in X$ in each round and directly observes its cardinal reward $u_t$,
which is assumed to be drawn from a latent but fixed distribution attached to $x_t$.
The set $X$ in the traditional MAB game is of finite cardinality $K$.  In more general settings~\cite{DBLP:conf/colt/DaniHK08,DBLP:conf/nips/MannorS11}, this set can be infinite but structured in some way. \citet{DBLP:conf/colt/DaniHK08}, for example, assume a stochastic setting in which $X$ is a convex, bounded subset of $\R^n$, and the expectation $\mu(x)$ of the corresponding value distribution is $\dotprod{\mu}{x}$, where $\mu\in \R^n$ is an unknown coefficient vector and $\dotprod{\cdot}{\cdot}$ is the inner product with respect to the standard basis. We refer to this as the \emph{linear expected utility setting}. We study here both the finite setting and the infinite setting.  

{\bf Main results.}
We provide general reductions from UBDB to MAB.   More precisely, we use a MAB strategy as a black-box for helping us play the UBDB game.  The art is in exactly how to use a black-box designed for MAB in order to play UBDB.
We present one method, $\Doubler$ (Section~\ref{sec:infinite})  which adds an extra  $O(\log T)$ factor to the expected regret function compared to that of the MAB black-box, assuming the MAB black-box has polylogarithmic (in $T$) regret, where $T$ is the time horizon. When the MAB black-box has polynomial regret, only an extra $O(1)$ factor is incurred.  This algorithm works for infinite bandit spaces.
We also present a reduction algorithm $\MultiSbm$ (Section~\ref{sec:stochastic_finite}) 
which works for finite bandit spaces and gives an $O(\log T$) regret, assuming
the MAB black-box enjoys an $O(\log T)$ expected  regret function with some mild higher moment assumptions.  These assumptions are satisfied, for example, by the seminal UCB algorithm \cite{Auer:2002:FAM:599614.599677}.
Our analysis in fact shows that for sufficiently large $T$, the regret of $\MultiSbm$ is asymptotically identical to that of UCB not only in terms of the time horizon $T$ but in terms of second order terms such as the differences between the values of the arms; it follows that $\MultiSbm$ is asymptotically optimal in the second order terms as well as in $T$.
Finally, we propose a third algorithm $\DoubleSbm$ (Section~\ref{sec:double sbm}) which we conjecture
to enjoy regret bounds comparable to those of the MAB algorithms hiding in the black boxes it uses.  We base the conjecture
on arguments about  a related, but different problem.  In experiments (Section~\ref{sec:experiments}) comparing our reductions with special-purpose UBDB algorithms,
 $\DoubleSbm$ performs clearly the best, further supporting our conjecture.

 All  results in this extended abstract assume the linear link function (see Section~\ref{sec:def}), but we also show preliminary results for other interesting link functions in Appendix~\ref{sec:more_general}.

{\bf Contributions in relation to previous work.} While specific algorithms for specific cases of the Dueling Bandits problem already exist \cite{DBLP:journals/jcss/YueBKJ12,DBLP:conf/icml/YueJ11,Yue/Joachims/09a}, our reductions provide a general approach to solving the UBDB. In particular, this paper provides general reductions that make it possible to transfer the large body of MAB work on exploiting structure in $X$ to the dueling case in a constructive and algorithmic way. Second, despite the generality of the reductions their regret is asymptotically comparable to the tournament elimination strategies in \cite{DBLP:journals/jcss/YueBKJ12,DBLP:conf/icml/YueJ11} for the finite  case as $T\rightarrow \infty$, and better than the regret of the online convex optimzation algorithm of \cite{Yue/Joachims/09a} for the infinite case (albeit in a more restricted setting).

In our  setting, the reward and feedback of the agent playing the online game are, in some sense, \emph{orthogonal} to each other, or \emph{decoupled}. A different type of decoupling was also considered in Avner et al.'s work  \cite{DBLP:conf/icml/AvnerMS12}, although this work cannot be compared to theirs.  
There is yet more work on bandit games where the algorithm plays two bandits (or more) in each iteration, e.g. Agarwal et al. \cite{DBLP:conf/colt/AgarwalDX10}, although there the feedback is cardinal and not relative in each step.
There is much work on learning from example pairs \cite{HerbrichGO00,DBLP:journals/jmlr/FreundISS03,DBLP:journals/jmlr/AilonBE12} as well as noisy sorting \cite{Karp:2007:NBS:1283383.1283478,Feige:1994:CNI:196751.196817}, which are not the setting studied here.
Finally, our results connect multi-armed bandits and online optimization to the classic econometric theory of discrete choice, with its use of preferential or choice information
to recover values of goods (see \cite{KTrain09} and references therein).

Another important topic related to our work is that of \emph{partial monitoring games}. The idea was introduced by \cite{piccolboni2001discrete}.  The objective in partial monitoring is to choose at each round an action from some finite set of actions, and receive a reward based on some unknown function chosen by an oblivious process. The observed information   is defined as some (known) function of the chosen action and the current choice of the oblivious process. 
One extreme  setting in which the observed information equals the reward captures MAB.  In the other extreme, the observed
information equals the entire vector of rewards (for all actions), giving rise to the so-called \emph{full information} game.  
Our setting is a strict case of partial monitoring as it falls in neither extremes.
Most papers dealing with partial monitoring either discuss non-stochastic settings or present problem-independent results. In both cases the regret is lower bounded by $\sqrt{T}$, which is inapplicable to our setting (see \cite{antos2012toward} for a characterization of partial monitoring problems). Bart{\'o}k et al. \cite{bartok2012adaptive} do present problem dependent bounds.  
Using their work, a logarithmic (in $T$) bound can be deduced for the dueling bandit problem, at least in the finite case. 
However, the dependence on the number of arms is quadratic, whereas we present a linear one in what follows.
Our algorithms are also much simpler and directly  take  advantage of the structure of the problem at hand.

\section{Definitions}\label{sec:def}

The set of actions (or \emph{arms}) is denoted by $X$. 
In a standard stochastic MAB (multi-armed bandit) game, each bandit $x\in X$ has an unknown associated expected utlity $\mu(x) \in [0,1]$.
At each step $t$ the algorithm chooses some $x_t\in X$ and receives from ``nature'' a random utility
 $u_t \in [0,1]$,  drawn from a distribution of expectation $\mu(x_t)$.  This utility is viewed by the algorithm.\footnote{It is typically assumed that this distribution depends on $x_t$ only, but this assumption can be relaxed.}
The regret at time $T$ of an algorithm is defined as $R(T) = \sum_{t=1}^T (\mu(x^*) - u_t)$.
where $x^*$ is such that $\mu(x^*) = \max_{x\in X}\mu(x)$ (we assume the maximum is achievable).
Throughout, for $x\in X$ we will let $\Delta_x$ denote $\mu(x^*) - \mu(x)$ whenever we deal with MAB.
(We will shortly make reference to some key results on MAB in Section~\ref{sec:MAB algos}.)

In this work we will use MAB algorithms as black boxes.  To that end, we define a Singleton Bandit Machine (SBM)
as a closed computational unit with an internal timer and memory.  A SBM $S$ supports three operations: $\rst$, $\advnce$ and $\fdbck$.
 The reset operation simply clears its state.\footnote{We assume the bandit space $X$ is universally known
to all SBM's.}  The $\advnce$ operation returns the next bandit to play, and $\fdbck$ is used for simulating a feedback (the utility).
It is assumed that $\advnce$ and $\fdbck$ operations are invoked in an alternating fashion. 
For example, if we want to use a SBM to help us play a traditional MAB game we
first invoke $\rst(S)$, then invoke and  set $x_1\leftarrow\advnce( S)$,  we will play $x_1$ against nature and observe $u_1$ and then invoke $\fdbck(S,u_1)$.
We then invoke and set $x_2\leftarrow \advnce(S)$, then we'll play $x_2$ against nature and observe $u_2$, then invoke $\fdbck(S, u_2)$
and so on.  
For all SBM's $S$ that will be used in the algorithms in this work, we will  only invoke the operation $\fdbck(S,\cdot)$  
with values in  $[0,1]$.

In the  utility based dueling bandit game (UBDB), the algorithm chooses $(x_t, y_t)\in X\times X$ at each step, and a corresponding pair of  random utilities $(u_t, v_t) \in[0,1]$ are given rise to, but not observed by the algorithm.  
We assume $u_t$ is drawn
from a distribution of expectation $\mu(x_t)$ and $v_t$ independently from a distribution of expectation $\mu(y_t)$.
The algorithm observes a
\emph{choice} variable $b_t\in\{0,1\}$ distributed according to the law (\ref{defchoice}).
This random variable should be thought of as the outcome of a duel, or match between $x_t$ and $y_t$.  The outcome $b_t=1$  (resp. $b_t=0$) should be interpreted as ``$y_t$ is chosen' (resp. ``$x_t$ is chosen'').\footnote{ We have just defined a two-level model in which the distribution of the random variable $b_t$ is determined by the outcome two other random variables $u_t, v_t$.  For simplicity, the reader is encouraged
to assume that $(u_t, v_t)$ is deterministically $(\mu(x_t), \mu(y_t))$.  Most technical difficulties in what follows are already captured by this simpler case.}  The link function $\phi$, which is assumed to be known, quantitatively determines how to translate the utilities $u_t, v_t$ to winning probabilities.  
 The linear link function $\philin$ is defined by $$\Pr[b_t = 1 | (u_t,v_t)] = \philin(u_t, v_t) \eqdef  \frac{ 1+v_t - u_t} 2 \in [0,1]\ .$$

The \emph{unobserved} reward is   $\Umeant = (u_t+v_t)/2$, and the corresponding regret after $T$ steps is $\Rmean(T)  \eqdef\sum_{t=1}^T(\mu(x^*) - \Umeant)$,
where $x^* = \argmax_{x\in X} \mu(x)$. 
This implies that expected \emph{zero regret} is achievable by setting $(x_t, y_t) = (x^*, x^*)$.   In practice, these two identical alternatives would be displayed as one, as would naturally happen in interleaved retrieval evaluation \cite{Chapelle/etal/12a}.
It should be also clear that playing $(x^*, x^*)$ is pure exploitation, because the feedback is then an unbiased coin with zero exploratory information.

We also consider another form of (unobserved) utility, which is given as $\Uchoicet \eqdef u_t (1-b_t) + v_tb_t$.  We call this choice-based utility, since the utility that is obtained depends on the user's choice.
Accordingly, we define $\Rchoicet \eqdef\mu(x^*) - \Uchoicet$. In words, the player
 receives reward 
associated with either the left bandit or the right bandit, whichever was \emph{actually chosen}. The utility $\Uchoice$ captures
the user's experience after choosing a result.  In an e-commerce  system, $\Uchoice$ may capture \emph{conversion},
namely, the monetary value of the choice. 
Although both utility modelings $\Umean$ and $\Uchoice$ are well motivated by applications,
we avoid dealing with
choice based utilities and regrets for the following reason:
regret bounds with respect to $\Umean$ imply  similar regret bounds
with respect to $\Uchoice$.  
\begin{obs}\label{obs:equiv}
Assuming a  link function where $u>v$ implies $\phi(u,v)>1/2$, for any $x_t, y_t$, $\E[\Rchoicet|(x_t, y_t)]\leq \E[\Rmeant|(x_t, y_t)]$.
\end{obs}
(Due to lack of space, the  proof can be found in Appendix~\ref{sec:proof:obs:equiv}.)
The observation's assumption on the link function in words is: when presented with two items, the item with the larger utility is more likely to be chosen. This clearly happens for any reasonable link function.
We henceforth assume utility $\Umean$ and regret $\Rmean$ and will no longer make references to choice-based versions thereof.

\subsection{Classic Stochastic MAB: A Short Review } \label{sec:MAB algos}

We review some relevant classic MAB literature.
We begin with the well known UCB policy  (Algorithm~\ref{alg:ucb})  for MAB in the finite case.
\begin{algorithm}[t!] 
\caption{UCB algorithm for MAB with $|X|=K$ arms.  Parameter $\alpha$ affects tail of regret per action in $X$.}
\label{alg:ucb}
\begin{algorithmic}
	\STATE $\forall x \in X$, set $\hat{\mu}_x = \infty$
	\STATE $\forall x \in X$, set $t_x = 0$
	\STATE set $t=1$
	\WHILE {true}
		\STATE let $x$ be the index maximizing $\hat{\mu}_x + \sqrt{\frac{(\alpha+2) \ln(t)}{2t_x}}$
		\STATE play $x$ and update $\hat{\mu}_x$ as the average of rewards so far on action $x$; increment $t_x$ by 1.
		\STATE $t\leftarrow t+1$
	\ENDWHILE
\end{algorithmic}
\end{algorithm}
The commonly known analysis of  UCB  provides   expected regret bounds. For the finite $X$ case, we need  a  less known, robust guarantee bounding the probability of playing a sub-optimal arm too often. Lemma~\ref{lem:ucb robust} is implicitly proved in \cite{Auer:2002:FAM:599614.599677}. For completeness, we provide an explicit proof in Appendix~\ref{app:UCB robust}.

\begin{lem} \label{lem:ucb robust}
Assume $X$ is finite.  Fix a parameter $\alpha>0$.  Let
$ H \eqdef \sum_{x\in X\setminus\{x^*\}} 1/\Delta_x$.
When running the UCB policy (Algorithm~\ref{alg:ucb}) with parameter $\alpha$ for $T$ rounds the expected regret is bounded by $$2(\alpha+2)H\ln(T) + K\frac{\alpha+2}{\alpha} = O(\alpha H \ln T)\ .$$
Furthermore, lex $x\in X$ denote some suboptimal arm  and let $s \geq 4\alpha\ln(T)/\Delta_x^2$. Denote by $\rho_x(T)$ the random variable counting the number of times arm $x$ was chosen up to time $T$. Then
$ \Pr[\rho_x(T) \geq s] \leq  \frac{2}{\alpha} \cdot (s/2)^{-\alpha}$.
\end{lem}

For the infinite case, we will review  a well known setting and result which will later be used to highlight the usefulness of Algorithm~\ref{alg:infinite}  (and the ensuing Theorem~\ref{thm:infinite}).
 In this setting, the set $X$ of arms is an arbitrary (infinite) convex set in $\R^d$. Here, the player chooses at each time point a vector $x \in X$ and observes a stochastic reward with an expected value of $\dotprod{\mu}{x}$, for some unknown vector $\mu \in \R^d$.\footnote{Affine linear functions can also be dealt with by adding a coordinate fixed as 1.} This setting was dealt with by \citet{DBLP:conf/colt/DaniHK08}. They provide an algorithm for this setting that could be thought of as linear optimization under  noisy feedback. Their algorithm provides (roughly) $\sqrt{T}$ regret for general convex bodies and $\mathrm{polylog}(T)$ regret for polytopes. 
Formally, for general convex bodies, they prove the following.

\begin{lem}[\citealt{DBLP:conf/colt/DaniHK08}]
Algorithm \confballone  (resp. \confballtwo) of \citet{DBLP:conf/colt/DaniHK08}, provides an expected regret of $O\left( \sqrt{dT\log^3T} \right)$  (resp. $O\left( \sqrt{d^2 T\log^3 T} \right)$ ) for any convex set of arms.  
\end{lem}

In case  $X$ is a polytope with vertex set $V$
and there is a unique vertex $v^*\in V$ achieving $\max_{x\in X} \dotprod{\mu}{x}$, and any other vertex $v\in V$
satisfies the gap condition $\dotprod{\mu}{v} \leq \dotprod{\mu}{v^*} - \Delta$ for some $\Delta>0$,
we say we are in the $\Delta$-gap case.

\begin{lem}[\citealt{DBLP:conf/colt/DaniHK08}] \label{lem:confball gap}
Assume the $\Delta$-gap case.
Algorithm \confballone  (resp. \confballtwo) of \citet{DBLP:conf/colt/DaniHK08}, provides an expected regret of $O\left(\Delta^{-1}d^2\log^3 T \right)$  (resp. $O\left( \Delta^{-1}d^3\log^3 T \right)$ ).  
\end{lem}

\section{UBDB Strategy for Large or Structured $X$}\label{sec:infinite}

In this section we consider UBDB in the case of a large or possibly infinite set of arms $X$, and the linear link function. The setting where $X$ is large typically occurs when some underlying structure for $X$ exists through which it is possible to gain information regarding one arm via queries to another.
Our approach, called  $\Doubler$, is best explained by thinking of the UBDB strategy as a competition between two players, one controlling the choice of the left arm and the other, the choice of the right one. The objective of each player is to win as many rounds possible, hence intuitively, both players should play the arms with the largest approximated value. Since we are working with a stochastic environment it is not clear how to analyze a game in which both players are adaptive, and whether such a game would indeed lead to a low regret dueling match (see also Section~\ref{sec:double sbm} for a related discussion). For that reason, we make sure that at all times one player has a fixed stochastic strategy, which is updated infrequently.

We divide the time axis into exponentially growing epochs. In each epoch, the left player plays  according to some fixed (stochastic) strategy which we define shortly, while the right  one plays adaptively according to a strategy provided by a SBM. At the beginning of a new epoch, the distribution governing  the left arm changes in a way that mimics the actions of the right arm in the \emph{previous} epoch. 
 The formal definition of $\Doubler$ is given in  Algorithm~\ref{alg:infinite}.

\begin{algorithm}[t!]
\caption{ ($\Doubler$): Reduction for finite and infinite X with internal structure.}
\label{alg:infinite}
\begin{algorithmic}[1]
	\STATE $S \leftarrow $ new SBM over $X$ \label{line:SBM}
	\STATE ${\cal L} \leftarrow$ an arbitrary singleton in $X$
	\STATE $i\leftarrow 1$, $t\leftarrow 1$
	\WHILE {true}
		\STATE $\rst(S)$
		\FOR {$j=1...2^i$} \label{line:j}
			\STATE choose $x_t$ uniformly from ${\cal L}$
			\STATE $y_t \leftarrow \advnce(S)$
			\STATE play $(x_t, y_t)$, observe choice $b_t$ 
			\STATE $\fdbck(S, b_t)$
			\STATE $t\leftarrow t+1$
		\ENDFOR
		\STATE ${\cal L} \leftarrow$ the multi-set of arms played as $y_t$ in the last for-loop
		\STATE $i\leftarrow i+1$
	\ENDWHILE
\end{algorithmic}
\end{algorithm}

The following theorem bounds the expected regret of Algorithm~\ref{alg:infinite} as a function of the total number $T$ of steps and the expected regret of the SBM that is used.

\begin{thm}\label{thm:infinite}
Consider a  UBDB game over a set $X$.
Assume the SBM $S$ in Line~\ref{line:SBM} of $\Doubler$ (Algorithm~\ref{alg:infinite}) has an expected regret of $c \log^\alpha T$ after $T$ steps, for all $T$.  Then the expected
regret of $\Doubler$ is at most $2c\frac{\alpha}{\alpha+1}\log^{\alpha+1} T$.  If the expected regret of the SBM is bounded by some function $f(T)=\Omega(T^\alpha)$ (with $\alpha>0$),  then the expected
regret of $\Doubler$ is at most $O(f(T))$.
\end{thm}
The proof is deferred to Appendix~\ref{sec:proof:thm:infinite}.
By setting the SBM $S$ used in Line~\ref{line:SBM} as the algorithms \confballone or \confballtwo of \citet{DBLP:conf/colt/DaniHK08}, we obtain the following:
\begin{cor}\label{cor:finite}
Consider a  UBDB game over a set $X$. Assume that the SBM $S$ in Line~\ref{line:SBM} of $\Doubler$ is algorithm \confballtwo (resp.\ \confballone). If $X$ is a compact convex set, then the expected regret of $\Doubler$ is at most $O(\sqrt{dT\log^3(T)})$ (resp.\ $O(\sqrt{d^2T\log^3(T)})$). In the $\Delta$-gap setting (see discussion leading to Lemma~\ref{lem:confball gap}), the expected regret is bounded by $O\left( \Delta^{-1}d^2 \log^4(T) \right)$ (resp.\ $O\left( \Delta^{-1}d^3 \log^4(T) \right)$).
\end{cor}

\noindent
In the finite case, one may set the SBM $S$ to the standard UCB, and obtain:
\begin{cor}\label{cor:infinite}
Consider a  UBDB game over a finite set $X$ of cardinality $K$. Let $\Delta_i$ be the difference between the reward of the best arm and the $i$'th best arm. Assume  the SBM $S$ in Line~\ref{line:SBM} of $\Doubler$ is  UCB. Then the expected regret of $\Doubler$  is at most $O(H \log^2(T))$ where $H \eqdef \sum_{i=2}^K \Delta_i^{-1}$
\end{cor}

\paragraph{Memory requirement issues:} A possible drawback of  $\Doubler$ is its need to store the history of $y_t$ from the last epoch in memory, translating to a possible memory requirement of $\Omega(T)$. This situation can be avoided in many natural cases. The first is the case where $X$ is embedded in a real linear space and the expectation $\mu(x)$ is a linear function. Here, there is no need to store the entire history of choices of the left arm but rather the average arm (recall that here the arms are thought of as vectors in $\R^d$, hence the average is well defined).  Playing the average arm (as $x_t$) instead of picking an arm uniformly from the list of chosen arm gives the same result with memory requirements equivalent to storage of one arm. In other cases (e.g., $X$ is not even geometrically embedded) this cannot be done. Nevertheless, as long as we are in a $\Delta$-gap case, as $T$ grows,  the arm played as $y_t$ is the optimal one with increasingly higher probability.  In more detail, if the regret incurred in a time epoch is $R$, then the number of times a suboptimal arm is played is at most $R/\Delta$. As $R$ is polylogarithmic in $T$, the required space is polylogarithmic in $T$ as well.  We do not elaborate further on memory requirements and leave this as future research.

\section{UBDB Strategy for Unstructured $X$}\label{sec:stochastic_finite}

In this section we present and analyze an alternative reduction strategy, called $\MultiSbm$, particularly suited for the finite $X$ case where the elements of $X$ typically have no structure. $\MultiSbm$ will not incur an additional logarithmic factor
as our previous approach did.  Unlike the algorithms in \cite{DBLP:conf/icml/YueJ11, DBLP:journals/jcss/YueBKJ12}, we will  avoid running an elimination tournament, but just resort to a standard MAB strategy by reduction.
 Denote $K = |X|$. The idea is to use $K$ different SBMs in parallel,  where each instance is indexed by an element in $X$.
In step $t$ we choose a left arm $x_t\in X$ in a way that will be explained shortly. The right arm, $y_t$ is chosen 
according to the suggestion on the SBM indexed by $x_t$, and the binary choice is fed back to that SBM.  In the next round,  $x_{t+1}$ is set to be $y_t$, namely, the right arm becomes the left one in the next step. Algorithm~\ref{alg:K TBMs} describes $\MultiSbm$  exactly.

\begin{algorithm}[t!] 
\caption{($\MultiSbm$): Reduction for unstructured finite $X$ by using $K$ SBMs in parallel.}
\label{alg:K TBMs}
\begin{algorithmic}[1]
	\STATE For all $x\in X$:  $S_x \leftarrow $ new SBM over $X$, $\rst(S_x)$  \label{line:SBMsInit}
	\STATE $y_0 \leftarrow$ arbitrary element of $X$
	\STATE $t \leftarrow 1$
	\WHILE {true}
		\STATE $x_t \leftarrow y_{t-1}$
		\STATE $y_t \leftarrow \advnce(S_{x_t})$  \label{line:SBMsAdvance}
		\STATE play $(x_t, y_t)$, observe choice $b_t$ 
		\STATE $\fdbck(S_{x_t}, b_t)$
		\STATE $t\leftarrow t+1$
	\ENDWHILE
\end{algorithmic}
\end{algorithm}

Naively, the regret of the algorithm can be shown to be at most $K$ times that of a single SBM. However, it turns out that the regret is in fact asymptotically competitive with that of a single SBM, without the extra $K$ factor.
 Specifically, we show that the total regret is in fact dominated solely by the regret of the SBM corresponding to the arm with maximal utility. 
To achieve this, we assume that the SBM's implement a strategy with a certain robustness property that implies a bound not only on the expected regret,  but also on the tail of the regret distribution.  
 More precisely, an inverse polynomial tail distribution is necessary.
Interestingly, the assumption is satisfied by the UCB algorithm \cite{Auer:2002:FAM:599614.599677} (as detailed in Lemma~\ref{lem:ucb robust}). 
Recall that $x^*\in X$ denotes an arm with largest valuation $\mu(x)$, and that  $\Delta_x \eqdef \mu(x^*) - \mu(x)$ for all $x\in X$.    Assume $\Delta_x > 0$ for all $x\neq x^*$.\footnote{If this is not the case, our statements still hold, yet the proof becomes slightly more technical. As there is no real additional complication to the problem under this setting, we ignore this case.}

\begin{defn}\label{defn:robustness}
Let $T_x$ be the number of times a (sub-optimal) arm $x\in X$ is played when running the policy $T$ rounds.  A MAB policy is said to be \emph{$\alpha$-robust} when it has the following property: for all $s \geq 4 \alpha \Delta_x^{-2} \ln(T) $, it holds that
$ \Pr [T_x > s] < \frac{2}{\alpha} (s/2)^{-\alpha}$.
\end{defn}

Recall that as discussed in Section~\ref{sec:MAB algos}, in \citeauthor{Auer:2002:FAM:599614.599677}'s \citeyearpar{Auer:2002:FAM:599614.599677} classic UCB policy this property can be achieved by slightly enlarging the confidence region.

\begin{thm}\label{thm:alg2bound} 
The total expected regret of $\MultiSbm$ (Algorithm~\ref{alg:K TBMs}) in the UBDB game is 
\vspace{-.06in}
\begin{align*}
 O\!\left(\! H \alpha\ln T + 
 H\alpha\left(\!K\ln K\!+\!K\ln\ln T -\!\!\sum\nolimits_{x \neq x^*} \!\ln \Delta_x\!\right) \!\!\right)\!, 
\end{align*}

\vspace{-.11in}
assuming the policy of the SBMs defined in Line~\ref{line:SBMsInit} is $\alpha$-robust for $\alpha=\max(3, \ln(K)/\ln\ln(T))$. The robustness can be ensured by choosing the UCB policy (Algorithm~\ref{alg:ucb}) for the SBM with parameter $\alpha$.
\end{thm}
Note that achieving $(\alpha=3)$-robustness requires implementing a variant of UCB with a slight modification of the confidence interval parameter in each SBM. Therefore, if the horizon $T$ is large enough so that $\ln\ln T >  (\ln K)/3$, then the total
regret is comparable to that of UCB in the standard MAB game.

The proof of the theorem is deferred to Appendix~\ref{sec:proofthm2}.  The main idea behind the proof is showing that a certain
``positive feedback loop'' emerges: if the expected regret incurred by the right arm at some time $t$ is low, then there is a higher chance that $x^*$ will be played as the left arm at time $t+1$.  Conversely, if any fixed arm (in particular, $x^*$) is played very often as the left arm, then the expected regret incurred by  the right arm decreases rapidly.

\section{A Heuristic Approach} \label{sec:double sbm}

In this section we describe a heuristic called $\DoubleSbm$ for playing UBDB, which shows extremely good performance in our experiments.  Unfortunately, as of yet we were unable to prove performance bounds that explain its empirical performance.
$\DoubleSbm$ uses two SBMs, corresponding to \emph{left} and \emph{right}.
In each round the pair of arms is chosen according to the strategies of the two corresponding SBMs.  
 The SBM corresponding to the chosen arm receives a feedback of $1$ while the other receives $0$. The formal algorithm is described in Algorithm~\ref{alg:DoubleSbm}.

The intuition for this idea comes from analysis of an adversarial version of UDBD, in which it can be easily shown that the resulting expected regret of $\DoubleSbm$ is at most a constant times the regret of the two SBMs which emulate an algorithm for  adversarial MAB.
(We omit the exact discussion and analysis for the adversarial counterpart of UDBD in this extended abstract.)
We conjecture that the regret of $\DoubleSbm$ is asymptotically  bounded by the combined regret of the algorithms hiding in the SBM's, with (possibly) a small overhead.  Proving this conjecture is especially interesting for settings in which $X$ is infinite
and a MAB algorithm with polylogarithmic regret exists. Indeed, previous literature based on tournament elimination strategies does not apply to infinite $X$, and $\Doubler$ presented earlier is probably suboptimal due to the extra log-factor it incurs.

Proving the conjecture appears to be tricky due to the fact that  the left (resp. right) SBM does not see a stochastic environment, because its feedback depends on non-stochastic choices made by the right (resp. left) SBM.    
Perhaps there exist bad settings where both strategies would be mutually `stuck' in some sub-optimal state. We leave the analysis of this approach as an interesting problem for future research.  Our experiments will  nevertheless include $\DoubleSbm$.

\begin{algorithm}[t!] 
\caption{($\DoubleSbm$): Reduction to two SBMs.}
\label{alg:DoubleSbm}
\begin{algorithmic}[1]
	\STATE  $S_L,S_R\leftarrow$ two new SBMs over $X$
	\STATE  $\rst(S_L), \rst(S_R), t\leftarrow 1$
	\WHILE {true}
		\STATE $x_t \leftarrow \advnce(S_L)$;  $y_t \leftarrow \advnce(S_R)$
		\STATE play $(x_t, y_t)$, observe choice $b_t \in\{0,1\}$ 
		\STATE $\fdbck(S_L, {\bf 1}_{b_t = 0})$; $\fdbck(S_R, {\bf 1}_{b_t = 1})$
		\STATE $t\leftarrow t+1$
	\ENDWHILE
\end{algorithmic}
\end{algorithm}

\section{Notes}

\paragraph{Lower Bound:} Our results contain upper bounds for the regret of the  dueling bandit problem. We note that a matching lower bound, up to logarithmic terms can be shown via a simple reduction to the MAB problem. This reduction is the reverse of the others presented here: simulate a SBM by using a UBDB solver. It is an easy exercise to obtain such a reduction whose regret w.r.t.\ the MAB problem is at most twice the regret of the dueling bandit problem. It follows that the same lower bounds of the classic MAB problem apply to the UBDB problem.

\paragraph{Adversarial Setting:} One may also consider an adversarial setting for the UBDB problem. Here, utilities of the arms that are assumed to be constant in the stochastic case are assumed to change each round in some arbitrary way. We do not elaborate on this setting due to space constraints but mention that (a) a lower bound of $\sqrt{KT}$ matching that of the MAB problem is valid in the UDBD setting, and (b) the \DoubleSbm \ algorithm, when using SBM solvers for the adversarial setting, can be shown to obtain the same regret bounds of said SBM solvers.

\section{Experiments} \label{sec:experiments}

We now present several experiments comparing our algorithms with baselines consisting of the state-of-the-art \textsc{Interleaved Filter} (\InF) \cite{DBLP:journals/jcss/YueBKJ12} and \textsc{Beat The Mean Bandit} (\BTMB) \cite{DBLP:conf/icml/YueJ11}.  Our experiments are exhaustive, as we include scenarios for which no bounds were derived (e.g. nonlinear link functions), as well as the much more general scenario in which $\BTMB$ was analyzed \cite{DBLP:conf/icml/YueJ11}.

Henceforth, the  set $X$ of arms  is  $\{A,B,C,D,E,F\}$.  For applications such as  the interleaving search engines \cite{Chapelle/etal/12a}, $6$ arms is realistic.
We considered $5$  choices of the expected value function $\mu(\cdot)$ and $3$  link functions\footnote{To be precise, the actual expected utility vector $\mu$ was a random permutation of the one given in the table.  This was done to prevent initialization bias arising from the specific implementation of the algorithms.}\footnote{Note that in row 'arith', $\mu(2)..\mu(6)$ form an arithmetic progression, and  in row 'geom' they form a geometric progression.}.  

{ \footnotesize
\begin{center}
\begin{tabular}{|c|c|}
\hline 
linear & $\phi(x,y)= (1+x-y)/2$ \\
\hline
natural & $\phi(x,y)=x/(x+y)$ \\
\hline
logit & $\phi(x,y)=(1+\exp\{y-x\})^{-1}$ \\
\hline
\end{tabular}
\end{center}
\vspace{-0.14in}
%
%

\begin{center}
\begin{tabular}{|c|c|c|c|c|c|c|}
\hline 
Name  & $\mu(A)$ & $\mu(B)$ & $\mu(C)$ & $\mu(D)$ & $\mu(E)$ & $\mu(F)$ \\
\hline \hline
1good & $0.8$ & $0.2$ & $0.2$ &$0.2$ &$0.2$ &$0.2$ \\
\hline
2good & $0.8$ & $0.7$ & $0.2$ & $0.2$ & $0.2$ & $0.2$ \\
\hline
3good & $0.8$ & $0.7$ & $0.7$ & $0.2$ & $0.2$ & $0.2$ \\
\hline
arith & $0.8$ & $0.7$ & $0.575$  & $0.45$  & $0.325$ & $0.2$ \\
\hline
geom & $0.8$ & $0.7$ & $0.512$ & $0.374$ & $0.274$ & $0.2$ \\
\hline
\end{tabular}
\end{center}
}

For each $15$ combinations of arm values and link function we ran all 5 algorithms $\InF$, $\BTMB$, $\MultiSbm$, $\Doubler$, and $\DoubleSbm$ with random inputs spanning a time horizon of up to $32000$.

We also set out to test our algorithms in 
 a scenario defined in \cite{DBLP:conf/icml/YueJ11}. We refer to this setting as YJ.
  Unlike our setting, where choice probabilities are derived from (random)
latent utilities together with a link function, in YJ an underlying 
unknown fixed matrix $(P_{xy})$ is assumed, where $P_{xy}$ is the probability
of arm $x$ chosen given the pair $(x,y)$.   The matrix satisfies very mild constraints.
Following  \cite{DBLP:conf/icml/YueJ11}, define  $\epsilon_{xy} = (P_{xy} - P_{yx})/2$.
The main constraint is, for some unknown total order $\succ$ over $X$,  the imposition $x \succ y \iff \epsilon(x,y)>0$.
The optimal arm $x^*$ is maximal in the total
order.  The regret incurred by playing the pair $(x_t,y_t)$ at time $t$ is $\frac 1 2(\epsilon_{x^*x_t} + \epsilon_{x^* y_t})$.

The $\BTMB$ algorithm \cite{DBLP:conf/icml/YueJ11} proposed for YJ is, roughly speaking, a tournament elimination scheme, in which a working set of candidate arms is maintained. Arms are removed from the set whenever there is high certainty about their suboptimality.
Although the YJ setting is more general than ours,  our algorithms can be executed without any modification,
giving rise to an interesting comparison with $\BTMB$.
For this comparison, we shall use the same  matrix $(\epsilon_{xy})_{x,y\in X}$ as in  \cite{DBLP:conf/icml/YueJ11}, which was empirically estimated from an operational search engine.

{ \scriptsize
\begin{center}
\begin{tabular}{|c|c|c|c|c|c|c|c|} 
\hline 
		& $A$	& $B$	& $C$	& $D$	& $E$ & $F$	\\
\hline 
$A$ & $0$		& $0.05$	& $0.05$	& $0.04$	& $0.11$	& $0.11$	\\
\hline
$B$ & $-0.05$	& $0	$	& $0.05$	& $0.06$ 	& $0.08$	& $0.10$	\\
\hline
$C$ & $-0.05$	& $-0.05$	& $0	$	& $0.04$ 	& $0.01$	& $0.06$	\\
\hline
$D$ & $-0.04$	& $-0.04$	& $-0.04$	& $0 $		& $0.04$	& $0$	\\
\hline
$E$ & $-0.11$	& $-0.08$	& $-0.01$	& $-0.04$	& $0	$	& $0.01$	\\
\hline
$F$ & $-0.11$	& $-0.10$	& $-0.06$	& $0	$	& $-0.01$	& $0$	\\
\hline
\end{tabular} 
\end{center}
}
(Note that $x^*=A \succ B \succ C \succ D \succ E \succ F$.)

\paragraph{Experiment Results and Analysis}
Figure~\ref{fig:exp} contains the expected regrets of these described scenarios as a function of the $\log$ (to the base $2$) of the time, averaged over $400$ executions, with one standard deviation confidence bars. The experiments reveal some interesting results. First, the heuristic approach is superior to all others in all of the settings. Second, among the other algorithms, the top two are the algorithms $\InF$ and $\MultiSbm,$ where $\MultiSbm$ is superior in a wide variety of scenarios.

\begin{figure*}[ht!]

\begin{center}
\scalebox{0.62}{\includegraphics{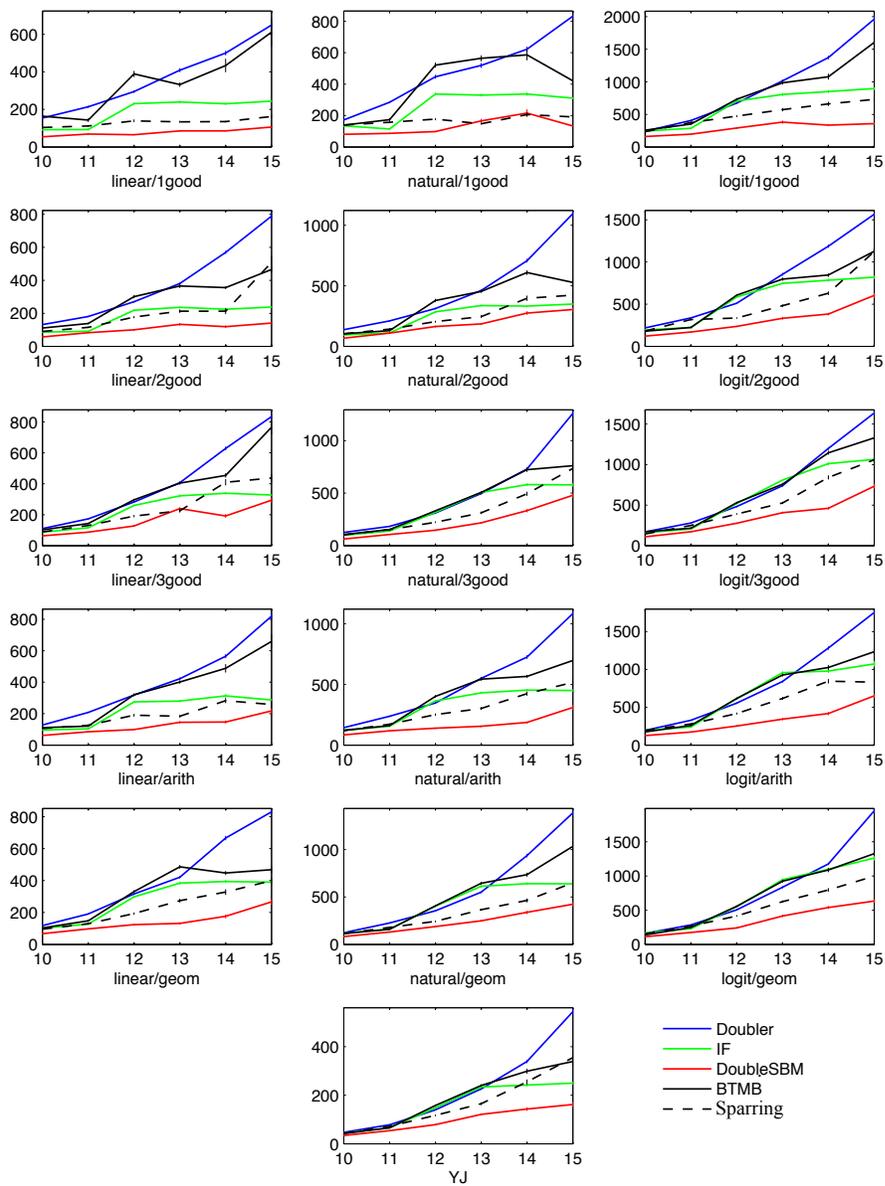}
}
\end{center}
\vspace*{-1.4cm}
\caption{
\small
Expected regret plots, averaged over $400$ runs for each of the 16 scenarios, and 5 algorithms.  The x-axis is the $\log$ to the base $2$ of the time, and the y-axis is the
regret, averaged over $400$ executions (with $1$ standard deviation confidence bars).
}
\vspace{-0.6cm}
\label{fig:exp}
\end{figure*}

\section{Future work}\label{sec:future} 
  We dealt with choice in sets of size $2$. What happens in cases where the player chooses from larger sets? We also analyzed only the linear choice function.
See Appendix~\ref{sec:more_general} for an extension of the results in Section~\ref{sec:stochastic_finite} to other link functions.

Both  algorithms $\Doubler$ and $\MultiSbm$ treated the left and right sides asymmetrically. 
This did not allow us  to consider distinct expected valuation functions for the left and right positions.
\footnote{Such a case is actually motivated in a setting where, say, the perceived user valuation of items appearing lower in the list are lower, giving rise to bias toward items appearing at the top.}
Algorithm $\DoubleSbm$ \emph{is} symmetric, further motivating the question of proving its performance guarantees.
 
 

%
Proving (or refuting) the conjecture in Section~\ref{sec:double sbm} regarding the regret of $\DoubleSbm$ is an interesting open problem.
Much like our proof idea for the guarantee of $\MultiSbm$, there is clearly a positive feedback loop between the
two SBM's in $\DoubleSbm$: the more often the left (resp. right) arm is played optimally, the right (resp. left) arm
would experience an environment which is closer to that of the standard MAB, and would hence incur expected regret approximately
that of the SBM it implements.




{


%
%

\newcommand{\placetextbox}[3]{
  \setbox0=\hbox{#3}
  \AddToShipoutPictureFG*{
    \put(\LenToUnit{#1\paperwidth},\LenToUnit{#2\paperheight}){\vtop{{\null}\makebox[0pt][c]{
\scalebox{0.7}{
\begin{tikzpicture}
\filldraw[white]
 (0,0) rectangle (2,0.4);
\filldraw[black] 
(0.8,0.2) node {#3};
\end{tikzpicture}
}}}}
  }%
}%



\section*{Acknowledgments}
The authors thank anonymous reviewers for thorough and insightful reviews.
This research was funded in part by NSF Awards IIS-1217686 and IIS-1247696, 
a Marie Curie Reintegration Grant PIRG07-GA-2010-268403, an Israel Science Foundation
grant  1271/33 and a Jacobs Technion-Cornell Innovation Institute grant.
%
}

\bibliographystyle{icml2014}
\bibliography{comparative_UCB,joachims, bettermab, cucb}
\appendix


\section{Robustness of the UCB algorithm} \label{app:UCB robust}
%

For completeness, we present a proof of robustness for the UCB algorithm, presented as Algorithm~\ref{alg:ucb} below. Note that we did not make an effort to bound the constants in the proof. We start by presenting Chernoff's inequality providing a tail bound for estimations of variables contained in $[0,1]$.

\begin{lem} \label{lem:chrnf 01}
Let $Y_1,\ldots,Y_t$ be i.i.d variables supported in $[0,1]$. Then for any $\eps>0$ it holds that
$$ \Pr\left[ \frac{1}{t}\sum_{i=1}^t Y_i - \E[Y_i] > \eps \right] \leq e^{-2t\eps^2} $$
\end{lem}

Recall that in our setting, there are $K$ arms, each with an expected reward. For convenience we assume the set of bandits $X$ is the set $\{1,\ldots,K\}$ and further assume for the purpose of the analysis that arm $1$ has the largest expected reward. We denote by $\Delta_i$ the difference between the reward of arm 1 and that of arm $i$.

\begin{proof} [Proof of Lemma~\ref{lem:ucb robust}]
For convenience, define $\beta=\alpha+2$ where $\alpha$ is the robustness parameter given as input to the algorithm. For $i>1$, define
$$ u_i(t) = 2\beta \ln(t)/\Delta_i^2 $$
If at time $t$, arm $i$ where $i>1$ (i.e. $i$ is suboptimal) was chosen, one of the following must be true
\begin{enumerate}
\item $\newtau_i(t) < u_i(t)$

\item $\hat{\mu_i} > \mu_i + \sqrt{\frac{\beta\ln(t)}{2\newtau_i(t)}} $

\item $\hat{\mu_1} +  \sqrt{\frac{\beta \ln(t)}{2\newtau_1(t)}} < \mu_1  $
\end{enumerate}
Here, $\newtau_i(t),\newtau_1(t)$ denote the number of times arms $i$ and $1$ (the optimal arm) were pulled up to time $t$.
Indeed, if all 3 are false we have 
$$ \hat{\mu_1} +  \sqrt{\frac{\beta \ln(t)}{2\newtau_1(t)}} \geq \mu_1 = \mu_i + \Delta_i \geq $$
$$  \mu_i + 2\sqrt{\frac{\beta\ln(t)}{2\newtau_i(t)}} \geq \hat{\mu_i} + \sqrt{\frac{\beta\ln(t)}{2\newtau_i(t)}}$$
and the $i$'th arm cannot be chosen.  Hence, denoting $\newtau_i(T)$ the number of times arm $i$ is queried in a total budget of $T$ queries, we have
$$ \E[\newtau_i(T) - u_i(T)] \leq \sum_{t=u_i(T)+1}^T \Pr[(2) \ \mathrm{or} \ (3)] $$
To bound the probability of event $(2)$ occurring, we use Chernoff's inequality (Lemma~\ref{lem:chrnf 01})
$$ \Pr[(2)] \leq \Pr\left[ \exists \newtau_i \in [t] : \ \hat{\mu_i} > \mu_i + \sqrt{\frac{\beta\ln(t)}{2\newtau_i}} \right] \leq $$
$$t\cdot t^{-\beta} = t^{1-\beta} .$$
The bound for event $(3)$ is analogous. It follows that the probability of events $(2)$ or $(3)$ occurring is bounded by $2t^{1-\beta}$ and
$$\E[\newtau_i(T) - u_i(T)]  \leq \sum_{t=u_i(T)+1}^T 2t^{1-\beta} \leq $$
\begin{equation} \label{eq:rob simple}
 \frac{2}{\beta-2} \left( 2 \beta \ln(T)  \Delta_i^{-2} \right)^{2-\beta} 
\end{equation}
Proving the bound on the expected regret is now a matter of simple calculation
$$\E[R] = \sum_{i>1} \E[\Delta_i \cdot \newtau_i(T)] \leq $$
$$\frac{2\sum_i \Delta_i}{\beta-2} + \sum_{i>1} 2\beta \ln(T)/\Delta_i \leq \frac{2K}{\beta-2} + 2\beta H\ln(T)$$

We proceed to prove the high probability bounds on the number of pulls of a suboptimal arm. Denote by $\newtau_i^s(T)$ the number of times arm $i$ was chosen starting from the time point $t \geq s$. Assuming $s\geq  2 \beta \ln(T)  \Delta_i^{-2}$, by the same arguments leading to equation \ref{eq:rob simple} we have
$$ \E[\newtau_i^s(T) - u_i(T)] \leq \frac{2}{\beta-2} s^{2-\beta} $$
Assume that arm $i$ was chosen at least $s$ times for some 
$$s \geq \frac{4(\beta+2)\ln(T)}{\Delta_i^{2}} $$
it follows that $\newtau_i^{s-u_i(T)-1}(T) \geq u_i(T)+1$. The probability of this happening is bounded by Markov's inequality by
$$ \Pr[ \newtau_i(T) \geq s ] \leq \Pr\left[\newtau_i^{s-u_i(T)-1}(T)-u_i(T) \geq 1 \right] \leq $$
$$\E\left[\newtau_i^{s-u_i(T)-1}(T) - u_i(T)\right] \leq $$ 
$$\frac{2}{\beta-2} (s-u_i(T)-1)^{2-\beta} \leq \frac{2}{\beta-2}  \left(\frac{s}{2}\right)^{2-\beta} $$
The last inequality holds since 
$$s \geq \frac{4(\beta+2)\ln(T)}{\Delta_i^{2}} \geq 2+  \frac{4\beta\ln(T)}{\Delta_i^{2}} = 2u_i(T)+2  $$
\end{proof}

\section{Proof of Theorem~\ref{thm:infinite}}\label{sec:proof:thm:infinite}

Let $B(T)$ denote the supremum of the expected regret  of the SBM S (defined in line\ref{line:SBM} of Algorithm~\ref{alg:infinite})
after $T$ steps, over  all possible utility distributions of the arm set $X$.

Fix a phase $i$ in the algorithm.  The length  $T_i$ of the phase is exactly $ 2^i$. 
For all time steps $t$ inside the phase, the left bandit $x_t$
is drawn from some fixed distribution.  Let $\mu'$ denote the common expectation $\E[u_t]=\E_{x_t}[u_t|x_t]$ of the reward of the left
arm in all steps $t$  in the phase.  Now, the SBM $S$ (defined in Line~\ref{line:SBM}) is playing  a standard MAB game
over the set $X$ with binary rewards.  Let $b_t$ denote the binary reward in the $t$'th step (within the phase).
By construction,
\begin{equation}\label{eq:ttttt}  \E[b_t|v_t, u_t] = \frac{v_t - u_t+1}{2} \in [0,1]\ .\end{equation}
By conditional expectation, for all $y\in X$, 
\begin{equation}\label{eq:ttttt1}  \E[b_t|y_t=y] = \frac{\mu(y) - \mu' +1}{2} \in [0,1]\ .\end{equation}

Note that the arm with highest expected reward is $y=x^*$.
By the definition of  the bound function $B(T)$, the total expected regret (in the traditional 
MAB sense) of the SBM $S$ in the phase is at most $B(T_i) = B(2^i)$.   This means, that
$$ \E\left [\sum_t \left (b_t - \frac{\mu(x^*) - \mu'+1}{2}\right)\right] \leq B(2^i)\ ,$$
where the summation runs over $t$ in the phase.  But this clearly means, using (\ref{eq:ttttt1}), that
$$ \E\left [\sum_t \frac{\mu(y_t)-\mu(x^*)}{2} \right] \leq B(2^i)\ .$$
But notice that $\E[v_t] = \E_{y_t}\E[v_t|y_t] = \E[\mu(y_t)]$.  Hence,
$$ \E\left [\sum_t \frac{v_t-\mu(x^*)}{2} \right] \leq B(2^i)\ .$$
In words, this says that the expected contribution of the \emph{right arm} to the regret (in the UDBD game) in phase $i$
is at most $B(2^i)$.
It remains to bound the expected contribution to the regret of the left bandit in phase $i$,
which is drawn by a distribution which assigns to all $x\in X$ a probability proportional to the frequency of $x$ 
played as the \emph{right arm} in the \emph{previous phase}.\footnote{If $X$ is infinite, to be precise we need to say that the distribution is also supported on the set of arms played on the right side in the previous phase.} By the principle of conditional expectation, and due to the linearity of the link function, the expected regret incurred by $x_t$ (in each step in the phase) is \emph{exactly} the average expected regret contributed by the right bandit in phase $i-1$,  and hence at most $B(2^{i-1})/2^{i-1}$.  This means that the total expected regret incurred by the left bandit in phase $i$ is bounded
by $2^i(B(2^{i-1})/2^{i-1}) = 2B(2^{i-1}).$
Concluding, for a time horizon of $T$ uniquely decomposable as  $ 2+4+8+\cdots + 2^k + Z$ for some  integers $k\geq 1$ and $0\leq Z\leq 2^{k+1}$-1, the total expected regret is
given by the following function of $T$:
\begin{equation}\label{J}  1/2 + 3B(2) + 3B(4) + \cdots  +  3 B(2^{k}) + B(Z)\ .\end{equation}

The theorem claim is now obtained by simple analysis of (\ref{J}).  

\section{Proof of Theorem~\ref{thm:alg2bound}}\label{sec:proofthm2}

To follow the proof, it is important to understand that in  $\MultiSbm$ (Algorithm~\ref{alg:K TBMs}),  exactly one SBM is advanced at each step
in Line~\ref{line:SBMsAdvance}.
This means that the internal timer of each SBM may be (and usually is) strictly behind the iteration counter of the algorithm,
which is measured by the variable $t$.
Denote by $\newtau_x(t)$ the total number of times $S_x$ was advanced after $t$ iterations of the algorithm,
for all $x\in X$.

We now assume that all coin tosses are fixed (obliviously) in advance.  This allows us to discuss the regret of the SBM $S_x$ (line~\ref{line:SBMsInit}) after $T'$ \emph{internal} steps even if in practice the value $t$ for which $\newtau_x(t)=T'$
might be much larger than the total number of arm pulls $T$ , and in fact, may not even  exist.

Notice that internally, $S_x$ sees a world in which the reward is binary, and the expected reward 
for bandit $y\in X$ is exactly $(\mu(y) - \mu(x) + 1)/2$ at each internal step.    This is because
 when $S_x$ is advanced, the
left bandit (in the UBDB game) is identically $x$. It follows that in all SBMs, the suboptimalities are the same and are $\Delta_y/2$ for arm $y$.

For $x\in X$ and integer  $T'>0$, let
$$ R_x(T') = \frac 1 2\sum_{t: \newtau_x(t) \leq T', x_t=x}\Delta_{y_t}$$
In words, this is the contribution of the right bandit choices to the UBDB regret at all times $t$ for which the
left bandit is chosen as $x$, and $S_x$'s internal counter has not surpassed $T'$.
The expression $R_x(T')$ , by the last discussion, also measures the expected internal regret seen by $S_x$ after
$T'$ internal steps.
Similarly, we define
$$ R_{xy}(T') = \#\{t: \newtau_x(t)\leq T', x_t = x, y_t=y\}\Delta_{y}/2$$
This measures a part of $R_x(T')$ for which the right bandit is $y$.
 We start with an observation expressing the regret of the entire process as a function of the different $R_{xy}$'s.
It will be useful to define $\rho_{xy}(T') = \#\{t: \newtau_x(t)\leq T', x_t = x, y_t=y\}$, so that $R_{xy}(T') = \rho_{xy}(T')\Delta_{y}/2$.
\begin{obs} 
For any $T\geq 1$, the total regret  $R(T)$ of $\MultiSbm$
after $T$ steps satisfies
$ \left| R(T) - 2\sum_{x\in X} \sum_{y\in X} R_{xy}(\newtau_x(T)) \right| \leq 0.5 $.
\end{obs}
We conclude that in order to bound the expected regret $R(T)$ it suffices to bound the expressions $\E[R_{xy}(\newtau_x(T))]$. By using the upper bound of $\newtau_x(T) \leq T$, we get the trivial bound  for $\E[R(T)]$ of $K$ times the expected regret of a single machine. The main insight is to exploit  the fact that typically, $\newtau_x(T)$ is  order of  $\ln  T$ for suboptimal $x$.  
We begin with the observation that for any fixed $x,y\in X$ ($x$ suboptimal), $s\geq 8 \alpha $,
\begin{align}\Pr[ &R_{xy}(T) \geq  (s \ln T) / \Delta_y  ]  \nonumber \\
&= \Pr[ R_{xy}(T) \geq  ((s/2) \ln T) / (\Delta_y/2)  ] \nonumber \\
&= \Pr[ \rho_{xy}(T)\geq  ((s/2) \ln T) / (\Delta_y/2)^2  ]  \nonumber\\
&\leq
\left((s/4)\ln T)/(\Delta_y/2)^2\right)^{-\alpha} \leq (s\ln T)^{-\alpha}\  \label{ttttt}
\end{align}
This is immediate from the 
$\alpha$-robustness of the SBM and the fact we choose $\alpha>2$. For the same assumption on $s$ and $x,y$ and using the union bound,
\begin{eqnarray}
\Pr\left[ \exists  p\in\{0,\dots, \lceil \ln\ln T\rceil\}:  R_{xy}\left(e^{e^p} \right) \geq  s \cdot p/\Delta_y   \right]  \nonumber \\ 
\leq  2s^{-\alpha}\   \label{gtt}
\end{eqnarray}
We now bound the quantity $\newtau_x(T)$ for any nonoptimal fixed $x$. Using the (trivial) fact that all $z\in X$ satisfy
$\newtau_{z}(T) \leq T$, together with the fact that SBM $S_x$ is advanced in each iteration only if $x$ was the right bandit in the previous one, we have that for all suboptimal $x$,
\begin{align}
&\Pr [ \newtau_x(T)  \geq (s K\ln T) /\Delta_x^2 ]   \nonumber \\
 &\leq \sum_{z \in X} \Pr\left[  R_{zx}(T)  \geq (s \ln T) /\Delta_x  \right]  \leq K/(s\ln T)^\alpha, \label{ttr} 
\end{align}
where the rightmost inequality is by union bound and (\ref{ttttt}).
Fix some $x,y \in X$ ($x$ suboptimal). The last two inequalities give rise to a random variable $Z$ defined as the minimal scalar for which we have 
$$ \forall T' \in [e,e^{e},e^{e^2},\ldots,e^{e^{\lceil \ln\ln(T) \rceil}}], $$
$$ \newtau_x(T) \leq  (Z K \ln T)  /\Delta_x^2, \ \ R_{xy}(T') \leq  (Z\ln T')/\Delta_y$$
By (\ref{gtt})-(\ref{ttr}) we have that for all  $s\geq 8 \alpha $, $Pr[Z\geq s] \leq 2s^{-\alpha} + K(s\ln T)^{-\alpha}$.
Also, conditioned on the event that $\{Z \leq s\}$ we have that 
$R_{xy}(\newtau_x(T)) \leq R_{xy}^s \eqdef   s\cdot e \cdot \ln( (s K \ln T ) /\Delta_x^2 )/\Delta_y$, which is
$O\left( s\Delta_y^{-1} \left( \ln \ln T+\ln K +\ln  s+\ln(1/\Delta_x) \right)  \right) $.
Combining, $ \E[R_{xy}(\newtau_x(T))]$ is bounded above by:
\begin{align}
&R_{xy}^{8 \alpha -1} + \sum_{i=0}^{\infty} R_{xy}^{8\alpha+i} (2(8\alpha+i)^{-\alpha} + K((8\alpha+i)\ln T)^{-\alpha})\ .\nonumber
\end{align}
For $\alpha=\max\{3,2+(\ln K)/\ln \ln T)\}$, it is easy to verify that the last expression converges to $O(R_{xy}^{8 \alpha })$, hence
$$ \E[R_{xy}(\newtau_x(T))] = O\left(  \alpha \Delta_y^{-1} \left( \ln\ln T +\ln K +\ln(1/\Delta_x) \right) \right).$$
Concluding, the total expected regret $\E[R]$  is at most $0.5 + \E[R_{x^*} + \sum_{x,y \in X\setminus \{x^*\}} R_{xy}]$,
 clearly proving the theorem.

\section{Extension to more General Models}\label{sec:more_general}
Assume the setting of Section~\ref{sec:stochastic_finite}.
In this section we assume for simplicity that for any $t$ and any choice of $x_t, y_t$, the utilities are deterministically $u_t=\mu(x_t), v_t = \mu(y_t)$.
In \cite{DBLP:conf/icml/YueJ11}, the dueling bandit problem is presented where a more relaxed assumption is made on the probabilities of the outcomes of duels. Each pair of arm $x,y$ is assigned a parameter $\Delta(x,y)$ such that the probability of $x$ being chosen when dueling with $y$ is $0.5+\Delta(x,y)$. It is assumed that there exists some order $\succ$ over the arms and the $\Delta$'s hold two properties.
\begin{itemize}
\item \emph{(Relaxed) Stochastic Transitivity}: For some $\gamma\geq 1$ and any pair $x^* \succ x \succ y$ we have $\gamma\Delta(x^*,y) \geq \max\{\Delta(x^*,x),\Delta(x,y)\}$. 

\item \emph{(Relaxed) Stochastic Triangle Inequality}: For some $\gamma \geq 1$ and any pair $x^* \succ x \succ y$ we have $\gamma\Delta(x^*,y) \leq \Delta(x^*,x)+\Delta(x,y)$. 
\end{itemize}

We have analyzed  $\MultiSbm$ (Algorithm~\ref{alg:K TBMs})
under the assumption that $\Delta(x,y)=(\mu(x)-\mu(y))/2$. It can be easliy verified that our proof holds for arbitrary $\Delta$'s under the following assumption:
\begin{itemize}
\item \emph{(Relaxed) Extended Stochastic Triangle Inequality}. For some $\gamma \geq 1$, and any pair $x,y$ (where it does not necessarily hold that $x \succ y$) it holds that $\gamma\Delta(x^*,y) \leq \Delta(x^*,x)+\Delta(x,y)$.
\end{itemize}

This property is clearly held for $\Delta(x,y)=(\mu(x)-\mu(y))/2$. However, it holds for a wider family of $\Delta$'s. For example, it holds for $\Delta(x,y)=\mu(x)/(\mu(x)+\mu(y))$, assuming all $\mu$'s are in the region $[1/\gamma,1]$. The effect of $\gamma$ to the regret is given in the following theorem:

\begin{thm}\label{thm:alg2bound_extended} 
Assume the probability for the outcome of a duel is defined according to $\Delta(x,y)$, where $\Delta$ has the Relaxed Extended Stochastic Triangle Inequality with parameter $\gamma$. The total expected regret of $\MultiSbm$ 
in the UBDB game is asymptotic to
$$\gamma H\alpha \left(K\ln(K)+K\ln\ln(T) +\sum_{x \in X\setminus\{x^*\}} \ln(1/\Delta_x)\right)  +$$
$$  \ln(T)  H \alpha  $$
assuming the invoked MAB policy is $\alpha$-robust for $\alpha=\max(3, \ln(K)/\ln\ln(T))$.
\end{thm}

Notice that $\gamma$ does not enter the summand of $\ln(T)$, meaning that for large values of $T$, the regret is unaffected by $\gamma$. We defer the proof of the theorem to the full version of the paper.

\section{Proof of Observation~\ref{obs:equiv}}\label{sec:proof:obs:equiv}
By definition,  
$$\E[\Rchoicet|(x_t, y_t)] = \mu(x^*) - \E[\Uchoicet|(x_t, y_t)] \ .$$ 
But now note that by the definition of the link function
and of $\Uchoicet$, 
$$ \E[\Uchoicet|(x_t, y_t)] = \phi(u_t,v_t)u_t + \phi(v_t,u_t)v_t \geq \frac{u_t+v_t}{2} $$
where we used the assumption that for $u>v$, $\phi(u,v)>1/2$. Now notice that the expression on the right is exactly $\E[ \Umean|(x_t, y_t)]$. Hence,
$$\E[\Rchoicet|(x_t, y_t)] \leq \mu(x^*) - \E[\Umeant|(x_t, y_t)] = \E[\Rmeant] \ .$$

\end{document}